\documentclass[twoside]{article}

\usepackage[accepted]{aistats2020}
\usepackage{quoting}

\usepackage[round]{natbib}

\usepackage[textsize=tiny]{todonotes}
\usepackage{algorithm, algorithmicx}
\usepackage{algpseudocode}

\algnewcommand\algorithmicinput{\textbf{Input:}}
\algnewcommand\INPUT{\item[\algorithmicinput]}

\algnewcommand\algorithmicoutput{\textbf{Output:}}
\algnewcommand\OUTPUT{\item[\algorithmicoutput]}

\usepackage{hyperref}       
\usepackage{url}            
\usepackage{booktabs}       
\usepackage{amsfonts}       
\usepackage{nicefrac}       
\usepackage{microtype}      
\usepackage{amsthm}

\usepackage{amsmath,amssymb}
\usepackage{graphicx}

\renewcommand{\P}{\mathbb{P}}

\def\X{{\mathcal X}}
\def\Y{{\mathcal Y}}

\def\G{{\mathcal G}}

\def\L{{\mathcal L}}
\newcommand{\one}{\mathbb{I}}
\def\wh{\widehat}
\newcommand{\cX}{\X}
\newcommand{\cG}{\G}
\newcommand{\half}{{\frac12}}
\renewcommand{\eqref}[1]{Eq.~(\ref{eq:#1})}
\newcommand{\hx}{\wh{x}}

\newtheorem{thm}{Theorem}

\newtheorem{lemma}[thm]{Lemma}

\newcommand{\thmref}[1]{Theorem~\ref{thm:#1}}
\newcommand{\lemref}[1]{Lemma~\ref{lem:#1}}
\newcommand{\secref}[1]{Sec.~\ref{sec:#1}}
\newcommand{\myalgref}[1]{Alg.~\ref{alg:#1}}

\newcommand{\lab}{\texttt{label}}
\newcommand{\nats}{\mathbb{N}}

\newcommand{\algknown}{\texttt{RobustDFF}}

\newcommand{\algstat}{\texttt{StRoDFF}}

\newcommand{\phicount}{\texttt{fcount}}
\newcommand{\handle}{\texttt{HandleMistake}}

\newcommand{\papertitle}{Robust Learning from Discriminative Feature Feedback}
\begin{document}

\twocolumn[

\aistatstitle{\papertitle}

\aistatsauthor{ Sanjoy Dasgupta \And Sivan Sabato }

\aistatsaddress{ Department of Computer Science and Engineering \\
  University of California, San Diego\\
  California, USA
  \And  Department of Computer Science \\
Ben-Gurion University of the Negev\\
Beer Sheva, Israel} ]

\begin{abstract}
Recent work introduced the model of {\it learning from discriminative feature feedback}, in which a human annotator not only provides labels of instances, but also identifies discriminative features that highlight important differences between pairs of instances. It was shown that such feedback can be conducive to learning, and makes it possible to efficiently learn some concept classes that would otherwise be intractable. However, these results all relied upon {\it perfect} annotator feedback. In this paper, we introduce a more realistic, {\it robust} version of the framework, in which the annotator is allowed to make mistakes. We show how such errors can be handled algorithmically, in both an adversarial and a stochastic setting. In particular, we derive regret bounds in both settings that, as in the case of a perfect annotator, are independent of the number of features. We show that this result cannot be obtained by a naive reduction from the robust setting to the non-robust setting.
\end{abstract}

\section{Introduction}\label{sec:intro}

There has been a growing interest in learning from data sets in which instances not only have labels but may also have some information about relevant features. One way to think about this is that the human annotator labels each instance and also tries to pick out one or two features of the instance that help to (weakly) explain this label. The hope is that this will (1) lead to better models being learned, (2) reduce the number of instances needed for learning, and (3) help pave the way for more explainable models.

For instance, early work in information retrieval~\citep{CD90} looked at a simple protocol in which a user who labels a document (as, say, ``sports'') also highlights one or two words (like ``goalie'') that are predictive of this label. Such feedback is not very costly, since the labeler is in any case reading the document, but can be very helpful with identifying relevant features in the high-dimensional space of words. Numerous variations of this idea have been explored for text and vision applications~\citep{CD90,RMJ05,DMM08,S11a,MSCPY18}. Some theoretical studies~\citep{PD17,VisotskyAtCh19} have also formalized such schemes and shown that, in some situations, they lead to markedly better sample complexity than would be achieved when learning from labels alone.

Another type of feature feedback, which has been explored in human-in-the-loop computer vision work~\citep{BWBSWPB10,ZCK15}, asks the human to provide features that {\it distinguish} between two instances: for instance, the feature ``stripes'' distinguishes a zebra from a horse. The idea is that this is more concrete than suggesting predictive features and might thus be easier for the annotator to do reliably, especially in a multi-class setting. A formal model of this process was recently suggested by \cite{DasguptaSa18}. In this protocol, termed \emph{discriminative feature feedback}, learning takes place in rounds of interaction, where in each round the learner makes a prediction on the current example, and provides a previous example as an ``explanation''. If the prediction is incorrect, the teacher provides the correct prediction, and a feature distinguishing the incorrect explanation from the current example. The precise protocol and its semantics are reviewed in Section~\ref{sec:prelim}. The work of \cite{DasguptaSa18} provides a learning algorithm that uses this type of discriminative feedback and gives a mistake bound for it. Interestingly, the richer feedback makes it possible to learn some concept classes, such as DNF (disjunctive normal form, OR-of-AND) formulas, that are known to be computationally hard to learn from labels alone.

However, a significant drawback of that work is that it assumes that the human teacher never makes mistakes when labeling points or providing discriminative features. This is unrealistic in practice. In this paper, we introduce a {\it robust} discriminative feature feedback setting, and provide two robust algorithms for learning in this setting. The first algorithm considers a fixed data set that contains some ``exceptions'': points on which the teacher can make arbitrary errors. If, for example, the learning task is to distinguish between mammals, reptiles, amphibians, and so on, then these exceptions might be animals like penguin or platypus, corner cases that tend to defy simple rules. The second algorithm is for a statistical setting in which points are drawn i.i.d.\ from some underlying distribution, and a constant fraction of them are exceptions. In both cases, we provide proofs of correctness and mistake bounds.

\textbf{Our contributions.}
Our first contribution (Section~\ref{sec:feedback-model}) is to formulate a noise model for discriminative feature feedback that allows the teacher to behave arbitrarily on some subset of instances.

Second, we show that although the work of~\cite{DasguptaSa18} could, in principle, handle these exceptions by treating them as correct and devising more complicated rules to accommodate them, this would result in a large increase in the complexity of the concepts being learned (Theorem~\ref{thm:numexp}). This, in turn, would lead to a large number of mistakes on the data set. To complete the argument, we provide a new lower bound on the best mistake bound obtainable in the perfect-annotation setting, as a function of representation size (Theorem~\ref{thm:lowerbound}). In particular, we show that if the number of features is unbounded, as allowed by the original discriminative feature feedback setting, then this attempt to handle mistakes leads to a vacuous mistake bound.

Finally, we provide two new algorithms for robust learning under discriminative feature feedback, first in an adversarial setting where the ordering of instances is worst-case (Section~\ref{sec:agnostic-adversarial}), and then in a stochastic setting where the instances are sampled from an underlying distribution (Section~\ref{sec:agnostic-stochastic}). In both cases, we provide mistake bounds in terms of the size of the concept being learned and the number, or fraction, of exceptions (Theorems~\ref{thm:mistakebound} and \ref{thm:statbound}), but without any dependence on the number of features.

\section{Preliminaries}
\label{sec:prelim}

\cite{DasguptaSa18} defined the discriminative feature feedback model and studied it in a perfect-annotation setting. Let $c^*$ be the target concept  to be learned, where $c^*$ is a mapping from the input space $\X$ 
to a finite label space $\Y$. 
The learner has access to a set of Boolean features $\Phi$ on $\X$, and expresses concepts in terms of these. 

It is assumed that $\X$ can be represented as the union of $m$ sets in some family of sets $\G = \{ G_1,\ldots,G_m\}$, $\X = G_1 \cup G_2 \cup \cdots \cup G_m$. This is the \emph{internal representation} of the teacher. The representation, which is unknown to the learner, satisfies the following properties:
\begin{itemize}
\item Each of the sets is pure in its label: for each $i$, there exists a label $\ell(G_i) \in \Y$ such that $\forall x \in G_i, c^*(x) = \ell(G_i)$
\item Any two sets $G_i, G_j$ with $\ell(G_i) \neq \ell(G_j)$ have a {\it discriminating feature}: there is some $\phi \in \Phi$ such that if $x \in G_i$, $\phi(x)$ is satisfied, and if $x \in G_j$, $\phi(x)$ is not satisfied.

\end{itemize}
No restrictions are placed on the number of possible features, which can even be infinite. Therefore, negations and logical combinations of features can also be used as discriminative features.

For any $x \in \X$, denote by $G(x) \in \cG$  some set containing $x$. If there are multiple such components, $G(x)$ is some fixed choice.
The interactive learning protocol for the noiseless model is as 
follows:
\begin{itemize}
\item A new instance $x_t$ arrives.
\item The learner supplies a prediction $\widehat{y}_t$, and an instance $\widehat{x}_t$ which was previously seen with that label (``an explanation'').
\item If the prediction is correct, no feedback is obtained.
\item If the prediction is incorrect, the teacher provides the correct label $y_t = c^*(x)$, and a feature $\phi$ that separates $G(x_t)$ from $G(\widehat{x}_t)$, that is
  \[
    \phi(x) =
    \begin{cases}
      \mbox{\tt true} & \mbox{if $x \in G(x_t)$,} \\
      \mbox{\tt false} & \mbox{if $x \in G(\widehat{x}_t)$.}
    \end{cases}
  \]
\end{itemize}
Here, a feature is any mapping from examples to $\tt true/false$, which can either be given explicitly as a coordinate of $x$, or be calculated from its representation. 
It is shown in \cite{DasguptaSa18} that a legal representation of size $m$ exists if and only if the concept $c^*$ can be represented by a DNF formula of a special form, which they call a ``separable-DNF''. \cite{DasguptaSa18} give an algorithm for this interaction model, which obtains a mistake bound of $m^2$, with no dependence on the number of available features $|\Phi|$. 

\section{A feedback model with mistakes}
\label{sec:feedback-model}

In this work, we propose an extension of the discriminative feature feedback
model to a model that allows mistakes.  First, note that any deterministic
labeling function (that is, one in which the same example always gets the same
label) can be modeled in the perfect-annotation model described above, since one can always model $\cG$ as a set of singletons, one for each example in the input stream. However, this is clearly unhelpful, as there
can be no generalization to unseen examples, and the number of mistakes that the algorithm makes cannot be bounded. In particular, the mistake bound of $m^2$ obtained in \cite{DasguptaSa18} is meaningless if $m$ is equal to the number of examples. In fact, as we show in \secref{re}, even a small number of adversarial changes to a perfect model can lead to an unreasonably large representation.

We thus propose to allow a trade-off between the number of components modeling the concept and the number of \emph{exceptions}, which are examples that deviate from the model.
In this setting, we assume as above that there are $m$ components. However, instead of requiring that for all $x \in G_i$, $c^*(x) = \ell(G_i)$, we allow some exceptions. Formally, let
\begin{equation}
  M = M(c^*, \G) := \{x \in \X \mid c^*(x) \neq \ell(G(x))\}.
\label{eq:M}
\end{equation}
This is the set of exceptions which deviate from the representation $G_1,\ldots,G_m$. If the teacher provides a discriminative feature between a pair of examples that includes at least one exception, the feature might not be one that discriminates the respective components. In all other cases, the teacher behaves as in the perfect-annotation setting.

We study two cases: one in which the input is adversarial and $|M|$ is upper-bounded by some integer, and one in which the stream is an i.i.d.~draw from a distribution and the probability mass of $M$ is upper-bounded by some small value. An additional parameter that we consider is related to the amount of consistency among exceptions. Formally, for an example $\hx \notin M$ and a feature $\phi$, define
\begin{align}
M_{\hx,\phi} = \{ x \in M \mid \, &\phi \text{ is returned as the discriminating}\notag\\
&\text{feature between $x$ and $G(\hx)$}\}. 
\label{eq:Mxp}
\end{align}
For any $\hx, \phi$, we have $M_{\hx,\phi} \subseteq M$, thus one can always upper-bound $M_{\hx,\phi}$ using the size of $M$. However, in many cases it is more reasonable to assume that different exceptions would not generally use the same discriminative features, for instance if the exceptions are not the result of a coordinated corruption. Thus, we set a separate upper bound on $\max_{\hx \notin M,\phi}|M_{\hx,\phi}|$, which can be significantly smaller than $|M|$.

We make an additional technical assumption, which was not explicitly assumed in \cite{DasguptaSa18} where a perfect annotator was assumed: If the same
two components are separated  by the teacher more than once during the whole
interaction with the learner, then the same feature is provided in all of these interactions. Note that this requirement is always satisfied by some representation, if examples separated by different features are allocated to different components.

To conclude the definition of the setting, observe that on top of exceptions as defined above, the teacher can deviate from the interactive protocol in other ways. For instance, it can provide a feature $\phi$ that does not actually separate the two provided examples, or it can flag the same label on the same example first as a correct label and then later as a wrong label, violating the assumption of a deterministic labeling function. However, these types of inconsistencies can be easily identified when the feedback is provided, and ignored by the learner. Thus, for simplicity, we assume below that no such inconsistencies occur. Another type of deviation from the protocol can occur if the teacher provides a feature that does not actually separate the two components $G(x)$ and $G(\wh{x})$ (although it does separate $x$ and $\wh{x}$). This type of exception can be handled the same as exceptions in $M$. In summary, all exceptions are either easy to identify immediately, or covered by the current exception model.

\section{Exceptions under the perfect-annotation model}
\label{sec:re}

As discussed above, any deterministic labeling, including one with exceptions as defined above, can be modeled by the perfect-annotation setting, for instance by creating a special group $G_i$ for each exception, and dissecting other groups to make sure that the discriminative-feature property holds. In this section, we show that nonetheless, attempting to reduce a model with mistakes to a perfect-annotation model can result in a very large mistake bound when the number of possible features is large. First, we provide upper and lower bounds on the number of components required for such a reduction. 

By a {\it representation} $\G$, we mean a family of sets $\G = \{G_1, G_2, \ldots \}$ that cover $\cX$ and a labeling $\ell(G_i)$ of each set. The {\it size} of the representation is $|\G|$. Recall from (\ref{eq:M}) that $M(c, \G)$ denotes the set of exceptions for a given concept $c$ and representation $\G$.
\begin{thm}\label{thm:numexp}
Let $\G$ be a representation of size $m$.  Let $\bar{c}$ be a concept with $k$ exceptions, that is $|M(\bar{c}, \G)| = k$. 

Let $\bar{\G}$ be a representation of a minimal size $\bar{m}$ such that $|M(\bar{c},\bar{\G})| = 0$. Let $d = |\Phi|$ be the number of available features. Then:
\begin{enumerate}
\item[(a)] $\bar{m} \leq m+ d k$. 
\item[(b)] There exists a case in which $m = 1$ while $\bar{m} \geq d+1$. 
\end{enumerate}
\end{thm}

The proof is provided in the supplementary material. We remark that the bound in the theorem above 
is intimately related to the \emph{DNF exception problem}, which studies how many clauses are required to represent a concept defined by a DNF of a certain size with a bounded list of exceptions. This problem has been studied in several works \citep{Zhuravlev85, Kogan87, MubayiTuZh06, Maximov13}, including in the context of active learning with membership queries \citep{AngluinKr94,AngluinKrSlTu97}; however, tight upper and lower bounds are not known for this problem.

What is the significance of the representation size? The algorithm of \cite{DasguptaSa18} for the perfect-annotation setting makes $\Theta(m^2)$ mistakes, where $m$ is the representation size. However, they do not answer the question whether the order of this mistake bound is optimal. The following lower bound shows that it is, implying that the representation size is a crucial property. In particular, combined with \thmref{numexp}, it follows that reducing the setting which allows mistakes to the perfect-annotation setting when the number of features is unbounded would result in a vacuous mistake bound.

\begin{thm}
If feature feedback is given with respect to a representation of size $m$, then any algorithm must have a mistake bound $\Omega(m^2)$ in the perfect-annotation setting.
\label{thm:lowerbound}
\end{thm}

The proof is provided in the supplementary material. We have thus shown that a reduction of the setting with mistakes to the perfect-annotation setting results in a mistake bound that depends on the number of features $d$, which can be unbounded. In the next section we propose a robust algorithm which allows mistakes, and obtains an improved mistake bound, which does not depend on $d$. 

\section{Robust feature feedback in an adversarial setting}
\label{sec:agnostic-adversarial}

In this section, we derive a robust algorithm under an adversarial model. In this model,
there are no limitations on the input stream except that it conforms to the interaction protocol described in \secref{feedback-model}. In particular, the exceptions can appear at any arbitrary location in the stream. We assume that the number of exceptions (the size of $M$) is upper-bounded by $k$ for some integer $k$, and that for any $\hx\notin M$ and any $\phi$, $|M_{\hx,\phi}| \leq s$ for some integer $s \leq k$; recall the definitions (\ref{eq:M}) and (\ref{eq:Mxp}). We say that $s$ is an upper bound on the number of \emph{similar} exceptions. 

We propose an algorithm for this setting, called \algknown, and derive the following mistake bound for this algorithm.
\begin{thm}\label{thm:mistakebound}
  If there is a representation of size at most $m$ which satisfies the bounds of $k$ and $s$ defined above, then the number of mistakes made by  \algknown\ is at most
  
    $(m+k)((s+1)(m-1)+k+2)$, which is $O\big(\,((s+1)m+k)\cdot (m+k)\,\big).$

\end{thm}
Note that for $k = s = 0$, we retrieve the optimal mistake bound order of $O(m^2)$ for the perfect-annotation setting. Setting $s = k$ obtains a mistake bound of $O(km(m+k))$. Comparing this upper bound with the conclusions from \thmref{numexp} for the case $s = k$, it can be seen that a reduction to the perfect-annotation setting leads to a mistake bound of $O(m+dk)^2$. Thus, if $d \gg m$  then the mistake bound of \algknown\ is preferable.
Below, we present the algorithm and the mistake-bound analysis.

\subsection{Robust algorithm for the adversarial setting}

\begin{algorithm}[h]
  \caption{\algknown: Robust discriminative feature feedback for the adversarial setting}
  \label{alg:known}
  \begin{algorithmic}[1]
    \INPUT{Max.~components $m$, max.~exceptions $k$, max.~similar exceptions $s \leq k$.}
    \State $t \leftarrow 0$
    \State Get the label $y_o$ of the first example $x_o$
    \State Initialize $L$ to an empty list
    
    \While{true}
    \State $t \leftarrow t+1$
    \State get a new point $x_t$:
      \If{$\exists C[\wh{x}] \in L$ such that $x_t$ satisfies $C[\wh{x}]$}
         \State Predict $\lab[\wh{x}]$ and provide example $\wh{x}$
         \If{prediction is incorrect}
         \State Get correct label $y_t$ and feature $\phi$
         \State Update $\phicount[\hx]$, $C[\hx]$, $L$  by running:
         \State \hspace{1em}\handle$(m,\hx, k, s,\phi)$ (\myalgref{incorrect}).
         \EndIf
      \Else ~ (no relevant rule exists)
         \State Predict $y_0$ and provide example $x_0$
         \If{prediction is incorrect}
            \State Get correct label $y_t$ and feature $\phi$. 
            \State Add to $L$ an empty conjunction $C[x_t]$,
            \State \hspace{1em} and set $\lab[x_t] \leftarrow y_t$.
            \State Initialize $\phicount[\hx](\cdot)$ to $0$.
         \EndIf
       \EndIf
       \EndWhile
     \end{algorithmic}
     
   \end{algorithm}

   \begin{algorithm}[h]
  \caption{\handle: Handling an incorrect prediction for a given rule}
  \label{alg:incorrect}
  \begin{algorithmic}[1]
    \INPUT{Max.~components $m$, max.~exceptions $k$, max.~similar exceptions $s \leq k$, rule representative $\hx$, discriminating feature $\phi$, access to $\phicount,C,L$}
    \OUTPUT{Updates values of $\phicount[\hx],C[\hx],L$}
             \State Add $1$ to $\phicount[\wh{x}](\phi)$
             \If{$\phicount[\wh{x}](\phi) > s$}
             \State $C[\wh{x}] \leftarrow C[\wh{x}] \wedge \neg \phi$
             \State $\phicount[\wh{x}](\phi) \leftarrow 0$
             \If{$|C[\wh{x}]| \geq m$}
             \State delete $C[\wh{x}]$ from $L$
             \EndIf
             
             \Else

             \State $b \leftarrow m-1-|C[\wh{x}]|$
             \State \textbf{if} the sum of counters $\phicount[\hx](\phi)$ for all $\phi$ except for the $b$ largest counters is more than $k$ \textbf{then} remove $C[\wh{x}]$ from $L$.
             \EndIf
           \end{algorithmic}
         \end{algorithm}

   \algknown\ is listed in \myalgref{known}. It calls the procedure \handle, given in \myalgref{incorrect}. The algorithm maintains a set of conjunctions (rules) which are iteratively refined based on the feedback from the teacher. A rule is {\it created} if an example that matches none of the existing conjunctions appears. A rule is {\it refined} if mistakes with the feedback from the teacher warrants such a refinement. A rule may also be deleted.

   \algknown\ keeps track of the following information:
\begin{itemize}
  \item The first labeled example $(x_0,y_0)$. 
  \item A list of conjunctions $L$. 
\item For every conjunction $C[x] \in L$, its label, denoted $\lab[x]$
\item For every conjunction $C[x] \in L$, a mapping $\phicount[x]: \Phi \rightarrow \nats$ of counters, which count, for each feature, how many times it was provided by the teacher as a discriminating feature for $x$. Since $\Phi$ might not be finite, $\phicount[x](\phi)$ is only explicitly set when the counter is incremented for the first time. All uninitialized counters are treated as having a value of zero.
\end{itemize}

Exceptions might cause issues in rules in one of two ways: either a rule is created based on an exception, or it is wrongly refined based on one.
To avoid the latter, a rule based on a non-exception is only refined when there is at least one non-exception that warrants this specific refinement. This is guaranteed by collecting more than $s$ witnesses to a certain feature, before deciding on a rule refinement based on this feature.
Creating rules based on exceptions is not prevented in \algknown. Instead, the algorithm identifies rules that become too large, or have too many separating features, and removes them. We show in the analysis that this upper-bounds the number of mistakes that the algorithm makes due to rules based on exceptions, while keeping good rules intact.

 \subsection{Mistake bound for the adversarial setting}\label{sec:analysis}

We now prove \thmref{mistakebound}, the mistake bound of \algknown. We first prove several invariants of the algorithm. First, we prove that in rules representing components, these components are never split. 
\begin{lemma}\label{lem:gx}
At all times in the algorithm, if $\hx$ is not an exception then conjunction $C[\hx]$ is satisfied by every point in $G(\hx)$. In addition, for every literal $\phi$ in $C[\hx]$, there is some non-exception $x$ such that $G(x)$ is separated from $G(\hx)$ by $\phi$. 
\end{lemma}

\begin{proof}
  We prove the claim by induction on the length of $C[\hx]$.
  When $C[\hx]$ is first created, it is an empty conjunction so it is satisfied by all of $G(\hx)$.  When $C[\hx]$ is restricted by $\neg \phi$ in \handle, it means that $s+1$ examples were separated from $\hx$ by $\phi$. By the assumption that $|M_{\hx,\phi}| \leq s$, it follows that at least one of these examples, call it $x$, is not an exception, hence $G(\hx)$ is separated from $G(x)$ by $\phi$. This implies that $G(\hx)$ has no examples that are satisfied by $\phi$. Hence, after adding $\neg \phi$ to $C[\hx]$, the extended $C[\hx]$ is still satisfied by $G(\hx)$ and is separated by $\phi$ from $G(x)$.
 \end{proof}

 Next, we prove that two rules never represent the same component. 
 \begin{lemma}\label{lem:dup}
   For any two non-exceptions $x,x'$, if there are two rules $C[x]$ and $C[x']$ in $L$ then $G(x) \neq G(x')$. 
 \end{lemma}
 \begin{proof}
   Suppose $x$ is observed earlier in the input sequence and $x'$ is observed later; If $C[x]$ is generated and $C[x']$ is also generated, this means that $C[x]$, in its form when $x'$ is observed, does not satisfy $x'$. But by \lemref{gx}, $C[x]$ always satisfies $G(x)$. Hence, $x' \notin G(x)$, which implies the claim.
 \end{proof}

 Next, we prove that only rules created by exceptions might be deleted.
 \begin{lemma}\label{lem:delete}
   If \handle\ when run by \algknown\ deletes the rule $C[\hx]$,  then $\hx$ is an exception.
 \end{lemma}
 \begin{proof}
   Assume for contradiction that $\hx$ is not an exception but rule $C[\hx]$ is deleted. A rule can get deleted for one of two reasons. The first reason for deletion is if the conjunction $C[\hx]$ has at least $m$ literals.  Then, by \lemref{gx}, for each such literal in $C[\hx]$ there is some non-exception $x$ such that $G(x)$ is separated from $G(\hx)$ using that literal. Since there are $m$ components $G_i$, there are at most $m-1$ literals in $C[\hx]$, which is a contradiction to the size of $C[\hx]$. 
     The second reason for deletion is if the sum of the counters $\phicount[\hx](\phi)$ except for the largest  $b \equiv m-|C[\hx]|-1$ counters is more than $k$. Suppose that $\hx$ is not an exception. By \lemref{gx}, $|C[\hx]|$ components are already separated from it using literals in $C[\hx]$.  At most $b$ other components could have some overlap with $C[\hx]$. Thus, at most $b$ of the non-zero counters $\phicount[\hx](\phi)$ have a $\phi$ which separates $G(\hx)$ from some component that has an overlap with $C[\hx]$. All other counters must have been generated by exceptions, and the total number of such exceptions is at least the sum of the other counters. By the condition for deleting a rule, more than $k$ such exceptions were observed. But this contradicts the upper bound of $k$ for exceptions.

   In both cases, we reached a contradiction. Hence, $\hx$ is an exception.
 \end{proof}

 To bound the total number of mistakes, we first bound the total number of rules created by the algorithm.
\begin{lemma}
 \algknown\ creates at most $m+k$ rules. \label{lem:rgen}
\end{lemma}
\begin{proof}

  By \lemref{dup}, the total number of rules in $L$ generated by non-exceptions is at most the number of components, $m$.   Therefore, at most $m$ non-exception rules are ever generated. By \lemref{delete}, only rules generated by exceptions might be deleted.  Since rules are generated at most once for every input example, and there are at most $k$ exceptions in the input, at most $k$ rules generated by exceptions are ever generated.
\end{proof}

Next, we bound the number of mistakes associated with each rule.
\begin{lemma}\label{lem:mistakes}
  The number of mistakes resulting from examples that have been matched to a single rule $C[x]$ is at most $(s+1)(m-1) + k+1$.  
\end{lemma}
\begin{proof}
  For all $x,\phi$, at the end of each round of \algknown, $\phicount[x](\phi) \leq s$, since each new mistake that is matched to $C[x]$ increases some $\phicount[x](\phi)$ by $1$, and then, if $\phicount[x](\phi) = s+1$, zeros this counter and extends $C[x]$ by one. Therefore, for every feature that end up extending $C[x]$, there are at most $s+1$ mistakes on $C[x]$. Letting $r$ be the length of $C[x]$ after the last iteration in which it exists, this means that exactly $(s+1)r$ mistakes are matched with features that extend $C[x]$.
  
  The number of mistakes that do not match features that extend $C[x]$ is always at most $k + s(m-1-|C[x]|)$ at the end of an iteration, since if at any time during the run the sum of counters is increased beyond this number, it means that the sum of the counters except for the $m-1-|C[x]|$ largest ones is $k+1$, in which case the rule gets deleted.
  Also, whenever the rule is extended, one counter with value $s$ is zeroed, thus this property continues to hold. 
  Thus, the total number of mistakes for $C[x]$ is at most
  
    $(s+1)r+ k+1 + s(m-1-r) 
    \quad\leq s(m-1) + r + k + 1,$
  
  Since $r \leq m-1$, this proves the claim.
\end{proof}

\thmref{mistakebound} is now immediate, as follows: 
  Each rule makes at most $(s+1)(m-1)+k+1$ mistakes by \lemref{mistakes}. By \lemref{rgen}, at most $m+k$ rule are generated by \algknown. In addition, a mistake that does not match any rule creates a new rule, thus there are at most $m+k$ such mistakes. In total, \algknown\ makes at most $(m+k)((s+1)(m-1)+k+2)$ mistakes.

This concludes the analysis of the adversarial robust algorithm. In the next section, we study a robust algorithm for a stochastic setting. 

\section{Robust feature feedback in a stochastic setting}
\label{sec:agnostic-stochastic}

In this section, we assume that the stream is drawn from a stochastic source, with a probability of at most $\epsilon$ that a drawn example is an exception. In addition,  we assume that for all non-exceptions $\hat{x}$ and features $\phi$, the probability mass of $M_{\hx,\phi}$ is at most $\sigma \leq \epsilon$. The algorithm gets an additional confidence parameter $\delta$ as input, and guarantees are provided with a probability of $1-\delta$.

For a stream of a given size $n$, it is possible to apply \thmref{mistakebound} with $k \approx \epsilon n$ and $s \approx \sigma n$ to get a mistake bound for the stochastic setting. However, the resulting bound grows quadratically with the stream size, rendering it vacuous. Thus, we propose a different algorithm, called \algstat, and show that for this algorithm, the rate of mistakes for large stream sizes is bounded. We prove the following theorem.

\begin{thm}\label{thm:statbound}
  Let $\delta \leq 1/e^2$. Suppose that the exception rate is at most $\epsilon \leq \frac14$ and let the length of the stream of examples be $n$.  With a probability at least $1-\delta$, the rate of mistakes of \algstat\ on a stream of size $n$ is upper bounded by     \[
      O\Big((\sigma m + \epsilon )m\log(1/\delta)  + m^2\log^2(n/\delta)/\sqrt{n}\Big).
      \]

    \end{thm}

\subsection{Robust algorithm for the stochastic setting}
\newcommand{\tnew}{t_{\mathrm{lr}}}
\newcommand{\Nnew}{N_{\mathrm{lr}}}

\algstat\ is presented in \myalgref{stat}. The structure of \algstat\ is similar to that of \algknown, but some adaptations are required to take advantage of the stochastic assumption. The following additional information is stored by \algstat: $\tnew$ records the last time that a new rule was created. $\Nnew$ counts the number of examples that were not satisfied by a rule since round $\tnew$. $t(\hx)$ records the time that rule $C[\hx]$ was created, and $t(\hx,\phi)$ records the first time that an example with a discriminative feature $\phi$ was provided for the rule $C[\hx]$. In addition, \algstat\ uses the following functions: 
\begin{align}
  q(\epsilon,t) &:= \epsilon t + \frac{2}{3}\log(8t^3/\delta) + \sqrt{2\epsilon t \log(8t^3/\delta)},\\
  \gamma(\epsilon,r,t) &:= \frac{1}{1-2\epsilon}(r + 4\sqrt{r}\log^{3/2}(\frac{8t^2}{\delta})) - r+1. \label{eq:kappa}
\end{align}
These functions are used to calculate exception thresholds, in place of $k$ and $s$ that are used in \algknown.

A main difference between \algknown\ and \algstat\ is that in \algstat, not every example which is not satisfied by current rules causes the creation of a new rule. Instead, a rule is created only if a specific condition is met (see line \ref{newrule}). This condition compares the number of examples that fell outside $L$ since the last creation of a rule, to the number of examples that fell inside the rules. It is used to guarantee that rules are only created if there is sufficient probability mass outside current rules, thus bounding the number of rules created by exceptions.

\begin{algorithm}[ht]
  \caption{\algstat: Robust discriminative feature feedback for the stochastic setting}
  \label{alg:stat}
  \begin{algorithmic}[1]
    \INPUT{Max. components $m$, max.~prob.~of exceptions $\epsilon$, max.~prob.~of similar exceptions $\sigma$, confidence $\delta$}
    \State $t \leftarrow 0$; $\Nnew \leftarrow 0, \tnew \leftarrow 0$.
    \State Get the label $y_o$ of the first example $x_o$; 
    \State Initialize $L$ to an empty list

    \While{true}
      \State $t \leftarrow t+1$; get a new point $x_t$.
      \If{$\exists C[\wh{x}] \in L$ such that $x_t$ satisfies $C[\wh{x}]$}
        \State Predict $\lab[\wh{x}]$ and provide example $\wh{x}$
        \If{prediction is incorrect}
           \State Get correct label $y_t$ and feature $\phi$
           \State \textbf{if} $\phicount[\hx](\phi) = 0$, \textbf{then} $t(\hx,\phi) \leftarrow t$.
           \State $t' \leftarrow t - t(\hx,\phi)+1$.
           \State $n_s \leftarrow q(\sigma,t')+1$, $n_k \leftarrow q(\epsilon,t')$.
           \State Update $\phicount[\hx]$, $C[\hx]$, $L$ by running:
         \State \hspace{1em}\handle$(m,\hx, n_k, n_s, \phi)$.

         \EndIf
      \Else ~ (no relevant rule exists)
         \State Predict $y_0$ and provide example $x_0$
         \State $\Nnew \leftarrow \Nnew + 1$
         \If{prediction is incorrect}
         \State Get correct label $y_t$ and feature $\phi$.
         \If{$\Nnew \geq \gamma(\epsilon,t - \tnew - \Nnew+1, t)$ \label{newrule}}
           \State Add to $L$ an empty conj.~$C[x_t]$,
            \State \hspace{1em} and set $\lab[x_t] \leftarrow y_t$.
            \State Initialize $\phicount[\hx](\cdot)$ to $0$.
            \State $t(\hx) \leftarrow t$, $\Nnew \leftarrow 0$, $\tnew \leftarrow t$.
            \EndIf
         \EndIf
       \EndIf
       \EndWhile
     \end{algorithmic}
     
   \end{algorithm}

 \subsection{Error bound for the stochastic setting}\label{sec:stanalysis}

 In this section, we prove \thmref{statbound}.
First, we define the following events, which together guarantee the correctness of estimates based on $q(\cdot,\cdot)$ in the algorithm.
\begin{itemize}
\item $\xi_1 := \{$ At any time $t$ in \algstat, for any $t' \leq t$, the number of exceptions observed in the last $t'$ iterations is at most $q(\epsilon, t').$ $\}$. 
  \item $\xi_2 := \{$ At any time $t$ in \algstat, for any $t' \leq t$,  if  in round $t-t'+1$ a mistake was made and a feature $\phi$ separating $\hx$ was provided by the teacher, then the number of exceptions in $M_{\hx,\phi}$ observed afterwards, until iteration $t$ (inclusive), is at most $q(\sigma,t').$ $\}$. 
  \end{itemize}
  
  By Bernstein's inequality and a union bound on all the pairs $t' \leq t$,
  setting $\delta(t',t) := \delta/(4t^3)$, we get that
  $\xi = \xi_1 \wedge \xi_2$ holds with a probability at least
  $1-\delta/2$.

The proof of \thmref{statbound} is based on several lemmas. Some of the analysis is analogous to that of \algknown. However, upper-bounding the number of generated rules requires a new statistical analysis. We first give the lemmas that have direct analogs in the analysis of \algknown. The following lemma is analogous to \lemref{gx}.

\begin{lemma}\label{lem:gxstat}
Assume $\xi$. At all times during the run of \algstat, if $\hx$ is not an exception then $C[\hx]$ is satisfied by every point in $G(\hx)$. In addition, for every literal $\phi$ in $C[\hx]$, there is some non-exception $x$ such that $G(x)$ is separated from $G(\hx)$ by $\phi$. 
\end{lemma}
\begin{proof}
  The proof follows the same argument as the proof of \lemref{gx}, except that in \algstat, instead of waiting for $s+1$ examples, \handle\ restricts $C[\hx]$ by $\neg \phi$ if more than $n_s$ examples were separated from $\hx$ by $\phi$, where $n_s = q(\sigma, t-t(\hat{x},\phi)+1)+1$. By $\xi_2$, the number of exceptions in $M_{\hx,\phi}$ encountered since the first such example, which was encountered in round $t(\hx,\phi)$, is at most $n_s$. Therefore, at least one of the examples separated by $\phi$ is not an exception. The rest of the proof remains the same as the proof of \lemref{gx}. 
 \end{proof}

The following lemma is analogous to \lemref{dup}, proved above for \algknown.
\begin{lemma}\label{lem:dupstat}
  Assume $\xi$. In \algstat, for any two non-exceptions $x,x'$, if there are two rules $C[x]$ and $C[x']$ in $L$ then $G(x) \neq G(x')$. 
 \end{lemma}
 \begin{proof}
   The proof is identical to the proof of \lemref{dup}, except that it uses \lemref{gxstat} instead of \lemref{gx}. 
   
 \end{proof}

The following lemma is analogous to \lemref{delete}, proved above for \algknown.
 \begin{lemma}\label{lem:deletestat}
   Assume $\xi$. In \algstat, if a rule $C[\hx]$ gets deleted then $\hx$ is an exception.
 \end{lemma}
 \begin{proof}
   The proof is the same as that of \lemref{delete}, except that \lemref{gxstat} is used instead of \lemref{gx}. In addition, instead of the upper bound of $k$ on the number of exceptions which is used by \handle\ when running from \algknown, in the case of \algstat\ the upper bound in \handle\ on the maximal number of exceptions is set to $n_k := q(\epsilon, t-t(\hx) + 1)$. Thus, if the sum of the counters $\phicount[\hx](\phi)$ except for the largest  $b:= m-|C[\hx]|-1$ counters is more than $n_k$, then more than $n_k+1$ exceptions were observed since the creation of the rule $C[\hx]$ at time $t(\hx)$, which contradicts $\xi$. The rest of the proof is identical. 
 \end{proof}

In the next lemma, it is shown that rules are not created unless there is a significant probability mass outside the current rules. The proof of this lemma is provided in the supplementary material. The main idea of the proof is to show that the condition on line \ref{newrule} does not hold unless there is a sufficient probability mass outside the current set of rules. This is shown via a suitable concentration inequality, combined with an analysis of the dynamics of rule refinements in \algstat. 
\begin{lemma}\label{lem:epsilon}
Assume $\epsilon < \frac14$ and $\delta \leq 1/e^2$. With a probability at least $1-\delta/4$, all the rules generated by \algstat\ satisfy the following property: The probability mass of examples that fall outside of $L$ at the time the new rule is created is at least $2\epsilon$.
\end{lemma}

The next lemma upper-bounds the number of rules generated by \algstat. Crucially, unlike the case of \algknown, this number does not depend on the total number of exceptions, which is linear in the size of the stream in the stochastic setting. 
\begin{lemma}
  Assume $\epsilon < \frac14$ and $\delta \leq 1/e^2$. With a probability at least $1-\delta$, the total number of rules created by the algorithm is at most $R(m,\delta) := 4m\log(4/\delta)$. \label{lem:rgenstat}
\end{lemma}

\begin{proof}
  Assume that $\xi$ holds, which occurs with probability at least $1-\delta/2$. 
By \lemref{dupstat} the total number of rules in $L$ generated by non-exceptions is at most the number of components, $m$. Therefore, at most $m$ non-exception rules are ever generated. 
  To bound the number of rules created based on exceptions, we bound the probability, conditioned on a prefix of the stream, that the next rule created by \algstat\ after processing this prefix, is based on an exception. We use \lemref{epsilon}, which shows that with a probability at least $1-\delta/4$, a rule is created by \algstat\ only if the probability mass of examples that are not satisfied by any of the current rules is at least $2\epsilon$. Denote the event that the property in \lemref{epsilon} holds by $\xi_3$.   
  
  Under $\xi_3$, given that an example creates a new rule in round $t$, this is a random example from the set of examples not satisfied by the current set of rules $L$. Since the probability mass of exceptions is at most $\epsilon$, and the probability mass outside $L$ is at least $2\epsilon$, it follows that any new rule has a probability of at most a $\half$ to be based on an exception. 
Therefore, under $\xi_3$, the number of rules created until the next non-exception rule is created is an independent geometric random variable with a success probability of at least a $\half$. Moreover, at most $m$ rules are created based on non-exceptions. By \lemref{negbin}, which is provided in the supplementary material, the probability that more than  $R(m,\delta) := 4m\log(4/\delta)$ trials are required to obtain $m$ non-exception rules is less than $\delta/4$. Applying a union bound along with $\xi_3$ and $\xi$, the overall probability that this occurs is at least $1-\delta$.
\end{proof}

   The mistake bound for \algstat\ can now be proved. The proof is provided in the appendix in the supplementary material.

\section{Conclusion}
Discriminative feature feedback is a promising setting, which allows a more natural learning from a knowledgeable teacher. 
In this work, we showed that it is possible to learn with discriminative feature feedback even when the annotator is not perfect, and proved mistake bounds that do not depend on the number of features. We note that while the proposed algorithms require the problem parameters as inputs, this can be avoided by using a wrapper algorithm which searches for good parameter values. We defer the details to the long version of this work. The study of learning with rich feedback has the potential to be applicable to many real-life scenarios. In this work we have made an important step towards this goal. 

\paragraph{Acknowledgements}
 This research was supported by National Science Foundation grant CCF-1813160, and by a United-States-Israel Binational Science Foundation (BSF) grant no.~2017641. Part of the work was done while the authors were at the ``Foundations of Machine Learning'' program at the Simons Institute for the Theory of Computing, Berkeley.

\appendix

\section{Deferred Proofs}
\begin{proof}[Proof of \thmref{numexp}]

  For (a), we first observe that we may assume without loss of generality that the components in $\cG$ are pairwise disjoint: iteratively, for any two components $G_0,G_1$ that are not pairwise disjoint, replace them with $G'_0$, $G'_1$ such that, for $i \in \{0,1\}$,
  \[
    G'_i := (G_i\setminus G_{1-i}) \cup \{x \in G_0 \cap G_1 \mid G(x) = G_i\}.
    \]
    The result is a representation with the same number of components as $\cG$ that are pairwise disjoint, and all the responses of the teacher in the interaction protocol remain the same.
    
    Let $c^*$ be a concept that agrees with $\bar{c}$ on all but the $k$ exceptions, such that $|M(c^*, \G)| = 0$.
    We prove the upper bound by induction on $k$. 
  Suppose that for some value of $k$, for any concept $c'$ such that $|M(c',\G)| = k$, there is a representation $\G'$ of size $m' \leq m + d k$ that satisfies $|M(c', \G')| = 0$. This trivially holds for $k = 0$.
  
  Now, consider a concept $\bar{c}$ such that $|M(\bar{c},\G)| = k+1$. Let $c'$ be a concept which agrees with $c^*$ on all but $k$ elements, and agrees with $\bar{c}$ on all but one element. Let $\G' = \{G'_1,\ldots,G'_{m'}\}$ be the representation assumed by the induction hypothesis for $c'$, and let $x$ be the single element such that $\bar{c}(x) \neq c'(x)$. We construct a representation $\bar{\cG}$ for $\bar{c}$. 

  Under the disjointness assumption, there is a single component which includes $x$. Suppose it is $G'_1$. 
  For each $j \in [d]$, define the components $\bar{G}(j)$ as follows. Define $P^x_j := \{ z \in \cX \mid \phi_j(z) \neq \phi_j(x)\}$. Let $\bar{G}(j) := G'_1 \cap P_j^x$. Define an additional singleton component $\bar{G}_x = \{ x\}$. Note that $\{\bar{G}(j))\}_{j \in [d]} \cup \{\bar{G}_x\}$ exactly covers $G'_1$. 
Define   \[
    \bar{\G} := \{ \bar{G}(j)\}_{j\in [d]} \cup \{G'_2,\ldots, G'_{m'}\} \cup \{\bar{G}_x\}.
  \]
  For any $\bar{G} \subseteq G'_1$ such that $\bar{G} \neq \bar{G}_x$, set $\ell(\bar{G}) := \ell(G'_1)$. In addition, set $\ell(\bar{G}_x) := \bar{c}(x)$. $\bar{\G}$ is a legal representation, with $|M(\bar{\G}, \bar{c})| = 0$. The legality of $\bar{\G}$ can be observed by noting that the union of $\bar{\G}$ is $\X$, that the labels of all components agree with $\bar{c}$, and that any two components in $\bar{\G}$ with a different label can be separated by a single feature: If $\bar{G}_1 \subseteq G'_i$ and $\bar{G}_2 \subseteq G'_j$ for $i \neq j$ and their labels disagree, then the same feature that separates $G'_i$ and $G'_j$ separates $\bar{G}_1$ and $\bar{G}_2$. If $\bar{G}_1,\bar{G}_2 \subseteq G'_1$ and $\ell(\bar{G}_1) \neq \ell(\bar{G}_2)$, then necessarily one of the components is $\bar{G}_x$ and the other is $\bar{G}(j)$ for some $j$. In this case, the feature $j$ separates the two components. 
  The size of $\bar{\G}$ is $m' + d \leq m + d(k+1)$,
  as required by the upper bound. Note that while $\bar{\cG}$ is not pairwise disjoint, it can be converted to a pairwise-disjoint representation by the process described above. This completes the proof of the upper bound.

To prove the lower bound (b), it suffices to consider the following example, defined over $\cX = \{0,1\}^d$, where $\phi_j(x)$ is the value of coordinate $j$ in $x$.
Let $\cG = \{ \cX \}$, $\ell(\cX) = 0$.

Let $\bar{c}$ be a concept that agrees with $c^* \equiv 0$, except on $z_0 = (0,\ldots,0)$. Let $\cG'$ be a representation that has $|M(\bar{c}, \cG')| = 0$. We claim that $|\cG'| \geq d+1$. Consider the vectors $e_1,\ldots,e_d$. Suppose that some $G \in \cG'$ has $e_i,e_j \in G$ for $i \neq j$. Then no single feature can separate $G$ from the component that includes $z_0$. Therefore, there are at least $d$ components for each of $e_i$, and a separate one for $z_0$. This gives a lower bound of $d+1$.
\end{proof}

\begin{proof}[Proof of \thmref{lowerbound}]
  Let $P_m$ be the set of pairs $(i,j)$ such that $i, j \in [m]$ and $i < j$. Define a set of features $\Phi := \{ \phi_{i,j}^p \mid i,j \in [m], i \neq j, p \in \{0,1\}\}$. Define a family of $2^{|P_m|}$ possible representations $\{ \G_S\}_{S \subseteq P_m}$. The representation $\G_S$ includes $m$ components  $G_1,\ldots, G_m$, such that for $i < j$, component $G_i$ is separated from component $G_j$ using the feature $\phi_{i,j}^{S_{i,j}}$, where $S_{i,j} := \one[(i,j) \in S]$.  In other words, for each pair of components, one of two possible features $\phi_{i,j}^0, \phi_{i,j}^1$ separates them. We further define that in $G_i$ the separating feature is positive, while it is negative in $G_j$. For simplicity, we denote $\phi_{j,i} := \neg \phi_{i,j}$. 
  Formally, $G_i$ in representation $\G_S$ is the set of examples which satisfy

    $\left(\bigwedge_{j: i < j} \phi_{i,j}^{S_{i,j}}\right) \bigwedge \left(\bigwedge_{j: i > j} \neg\phi_{i,j}^{S_{i,j}}\right).$

  In all the representations, the label of the examples in $G_i$ is set to $i$.\footnote{A similar example with only two labels can be shown, at the cost of a smaller multiplicative factor in the mistake bound.}

  Define an example $x_{i,j}$ for $(i,j) \in P_m$ as follows: For all $l \neq i,j$ and $z \in \{0,1\}$, all the features $\phi_{i,l}^z$ and $\phi_{j,l}^z$ get the value that excludes them from $G_l$.
  The feature $\phi_{i,j}^0$ is set to positive, and $\phi_{i,j}^1$ is set to negative. Thus, in all representations $S$, $x_{i,j} \in G_i \cup G_j$, and $x_{i,j} \in G_i$ if and only if $(i,j) \in S$. 
  Now, consider a stream of examples that presents  $x_{i,j}$ for $(i,j) \in P_m$ in a uniformly random order and labels them using a representation $\G_S$ selected uniformly at random over $S \subseteq P_m$, so that the label of $x_{i,j}$ is $i$ if $(i,j) \in S$ and $j$ otherwise.

  The stream of examples is the same for all representations. Thus, the only information on $S$ can be obtained from the discriminative features. There are $\binom{m}{2}$ possible elements in $S$, and each discriminative feature feedback in this problem reveals whether $(i,j) \in S$ for a single pair $(i,j)$. Moreover, if this is unknown for some pair $(i,j)$ when $x_{i,j}$ is revealed, then both values of $S_{i,j}$ are equally likely conditioned on the run so far. In this case, any algorithm will provide the wrong label with a probability at least a half.
  Now, after less than $|P_m|/2$ mistakes, there is a probability of at least a half to observe such an example in the next iteration. Therefore, in the first $|P_m|/2$ examples of the stream, there is a probability of at least $1/4$ that the algorithm makes a mistake on the next example. Thus, the expected number of mistakes is at least $|P_m|/8 = \Omega(m^2)$.
\end{proof}

To prove \lemref{epsilon}, we use the following concentration inequality.
\begin{lemma}\label{lem:negbin}
  Let $\delta \in (0,1/e^2)$, let $k$ be an integer and let  $p \in [\half,1)$. 
  The probability that a sum of $k$ independent geometric random variables with probability of success $p$ is larger than $\frac{1}{p}\min(2k\log(1/\delta),(k+4\sqrt{k}\log^{3/2}(1/\delta)))$ is at most $\delta$.
\end{lemma}

\begin{proof}
  This lemma follows from Hoeffding's inequality, by noting that the number of successes in $N$ experiments with success probability $p$ is distributed as $\mathrm{Binom}(N,p)$, and having 
  \[
  \P[\mathrm{Binom}(N,p) < k] \leq \exp(-2N(p - k/N)^2).
  \]
  First, defining $N_1 := 2k\log(1/\delta)/p$, we have 
  \[
  k/N_1 = p/(2\log(1/\delta)) \leq p(1-1/\sqrt{2}).
  \]
  Hence, $p-k/N_1 \geq p/\sqrt{2}$. It follows that 
  \begin{align*}
  &\exp(-2N_1(p-k/N_1)^2) \leq \exp(-N_1p^2) \\
  &\quad\leq \exp(-N_1 p/2) = \exp(-k\log(1/\delta)) \leq \delta.
  \end{align*}

  Second, suppose that $k \geq 4\log(1/\delta)$, and let $\alpha := \sqrt{\log(1/\delta)/4k} \leq \frac14$. Defining \[
  N_2 := 2(1+4\alpha)k/p = \frac{1}{p}(2k + 4\sqrt{k\log(1/\delta)}),
  \]
  we have that 
  \[
  1/(p-\alpha) = 1/p + \alpha/(p(p-\alpha)) \leq (1 + 4\alpha)/p,
  \]
  where the last inequality follows since $p \geq \half$ and $\alpha \leq \frac14$. Therefore, $N_2 \geq k/(p-\alpha)$, hence \mbox{$k/N_2 \leq p-\alpha$}, hence 
  \begin{align*}
  &\exp(-2N_2(p-k/N_2)^2) \leq \exp(-4(k/p) \alpha^2)\\
  &\quad= \exp(-\log(1/\delta)/p)\leq \delta. 
  \end{align*}

  The proof is completed by observing that the first bound in the statement of the lemma is $N_1$, and the second bound is always larger than $N_2$, and for $k \leq 4\log(1/\delta)$, it is larger than $N_1$.
  \end{proof}
  
  We now prove \lemref{epsilon}.
\begin{proof}[Proof of \lemref{epsilon}]
  Denote by $L_t$ the set of rules $L$ at the end of round $t$ of the run of \algstat. Let \[
 \L_t = \{ x \in \cX \mid \exists C \in L_{t} \text{ such that } x \text{ satisfies } C\},
 \]
  and denote $p_t := \P[X \in \L_t]$, where $X$ is a random example drawn according to the distribution creating the input stream. We now prove the main claim: that with a high probability, a rule is not created by \algstat\ at round $t$ unless $p_{t-1} \leq 1-2\epsilon$. 
  The claim is proved by induction on the sequence of rules created by \algstat. For the basis of the induction,
  observe that $p_0 = 0$, since $L_0$ is empty. Therefore, the first rule created by \algstat\ certainly satisfies the claim for any $\epsilon < \half$. For the induction step, suppose that the claim holds for the first $l$ rules created by $\algstat$. Let $t_0$ be the round in which the $l$'th rule was created, and condition on the stream prefix ending in $t_0$.  We show that the next rule also satisfies the claim. 
  
  First, for any round $t \geq t_0$ until a new rule is created, $p_t$ is monotonic non-increasing. This is because the possible transformations, other than creating a new rule, are to restrict a rule or to delete a rule, both of which can never increase the set of examples covered by $L$. Therefore, if $p_{t_0} \leq 1-2\epsilon$, then regardless of the round $t$ in which the next rule is created, it satisfies $p_{t-1} \leq 1-2\epsilon$.  Thus, assume below that $p_{t_0} > 1 - 2\epsilon \geq \half$.
  $p_{t_0}$ is the probability that a random example observed immediately after round $t_0$ is satisfied by some rule in $\L_{t_0}$. Now, consider the first round after $t_0$ that an example in $\L_{t_0}$ arrives. Denote this round $t_1$. The value $T_1 := t_1 - t_0$ is a geometric random variable with a success probability $p_{t_0}$. By \lemref{negbin} with $k := 1$, $p := p_{t_0}$, with a probability at least $1-\delta/(8t_0^2)$, 
  \[
  T_1 \leq \frac{1}{p_{t_0}}(1+4\log^{3/2}(8t_0^2/\delta))) <   \gamma(\epsilon,1,t_0).
  \]
 In the last inequality we used $p_0 > 1-2\epsilon$ and the definition of $\gamma$. Assume below that this event holds. 
 
 Now, consider $N_{lr}$, which counts in \algstat\ the number of examples since the creation of the last rule, for which the default prediction $(x_0,y_0)$ was provided. These are the examples that were not satisfied by any rule in $L$ when they appeared. We prove by induction on the rounds that a new rule is not created at least until round $t_1$. If a new rule was not created until round $t \in \{t_0+1,\ldots,t_1-1\}$, then $L_t = L_{t_0}$ (since the set of rules does not change until $t_1$ when an example falls in $\L_{t_0}$).  In addition, $N_{lr} = t-t_0$, since the examples until round $t_1$ are not in $\L_t = \L_{t_0}$, thus they get the default prediction. Therefore, $t-t_0 - N_{lr} = 0$. It follows that in round $t$, 
 \[
 N_{lr} \leq T_1 < \gamma(\epsilon,1,t_0) \leq \gamma(\epsilon, t-\tnew - N_{lr} + 1, t).
 \]
 This means that the condition in line \ref{newrule} does not hold. Thus, under the event above, a new rule will not be created at round $t$. Since this holds by induction for all $t \in \{t_0+1,\ldots,t_1-1\}$, it follows that if $p_0 > 1-2\epsilon$ then a new rule is not created at least until the first example in $\L_{t_0}$ arrives.

  Now, $\L_{t_1}$ is the set of rules after this example arrives, and the probability mass of examples in $\L_{t_1}$ is $p_{t_1}$. More generally, let $t_i$ be the first round after $t_{i-1}$ in which an example in $\L_{t_{i-1}}$ appears. If no new rule is created between $t_0$ and $t_i$, then in round $t_i$, the set of rules changes from $L_{t_{i-1}}$ to $L_{t_i}$.  
  
  The number of rounds $T_i := t_i - t_{i-1}$ between each two such examples is a geometric random variable with success probability $p_{t_{i-1}}$. Let $r$ be the number of examples satisfied by $L$ which appear in the stream until the next rule after $t_0$ is created, and suppose for contradiction that $p_{t_r} > 1-2\epsilon$. 
  For $q \leq r$, define the random variable $S_q := \sum_{i=1}^{q} T_i$. This is a sum of $q$ independent geometric random variables, each with a probability of success larger than $1-2\epsilon$ (since $p_{t_{q}} \geq p_{t_r}$ for all $q \leq r$). Thus, $S_q$ is dominated by a sum of independent geometric random variables with a success probability of $1-2\epsilon$. Therefore, by \lemref{negbin}, with a probability at least $\delta/(8(t_0+q-1)^2))$, 
  \begin{align*}
  S_r &\leq \frac{1}{1-2\epsilon}(q+4\sqrt{q}\log^{3/2}(8(t_0+q-1)^2/\delta))\\
  &<   \gamma(\epsilon,q,t_0+q-1)+q-1.
  \end{align*}
  Assume below that this event holds for all $q \leq r$.
  We now prove that under the assumption on $p_{t_r}$, a new rule is not created until $t_r$, which is a contradiction. Suppose for induction that since round $t_0$ until round $t \leq t_r-1$, a new rule was not created. Let $q \leq r$ such that $t \in \{t_{q-1}+1,\ldots,t_{q}-1\}$. We have $t_q = t_0 + S_q$. Therefore, at round $t$, $N_{lr} = t-t_0 - (q-1) < S_q - (q-1)$. 
  It follows that under the assumed event, in round $t$
  \[
  N_{lr} < \gamma(\epsilon,q,t_0+q-1) \leq \gamma(\epsilon,t-\tnew-N_{lr}+1,t).
  \]
  Here, we used the fact that $t_0+q-1 \leq t$. It follows that the condition in line \ref{newrule} does not hold in round $t$, thus a new rule is not created in this round. By induction, this holds for all $t \leq t_r-1$, which contradicts the assumption that a rule was created until round $t_r$. Thus, if $p_{t_r} > 1-2\epsilon$ then a new rule is not created at least until round $t_r$. Since this analysis holds for any value of $r$, we conclude that if all the events above hold simultaneously, then a new rule is never created in round $t$ unless $p_{t-1} \leq 1-2\epsilon$. By a union bound on the created rules and the sequence of examples between rule-creations, this is true with a probability at least $1-\delta/4$.   
\end{proof}

\begin{proof}[Proof of \thmref{statbound}]
  First, we upper bound the number of mistakes on examples that are not satisfied by any rule when they are observed. Let $t_1,t_2,\ldots,t_R$, which sum to $n$, be the lengths of times between creations of new rules (where $t_1$ is time of the first rule and $t_R$ is the time between the last rule and the end of the stream). We have by \lemref{rgenstat} that $R \leq R(m,\delta)+1$.
  We have $1/(1-2\epsilon) = 1 + 2\epsilon/(1-2\epsilon) \leq 1 + 4\epsilon$, where the last inequality follows since $\epsilon \leq \frac14$. Hence,
  \begin{align*}
    \gamma(\epsilon,r,t) &\equiv \frac{1}{1-2\epsilon}(r + 4\sqrt{r}\log^{3/2}(8t^2/\delta)) - r+1\\
    &\leq 8\epsilon r + 8\sqrt{r}\log^{3/2}(8t^2/\delta).
  \end{align*}
  
  The number of mistakes resulting from examples not satisfied by any rule is upper-bounded by 
  \begin{align*}
    \sum_{i=1}^R \gamma(\epsilon,t_i, n) &\leq 8\epsilon n + 8\sum_{i=1}^R\sqrt{t_i}\log^{3/2}(8n^2/\delta)\\
    &\leq 8\epsilon n + 8\sqrt{Rn}\log^{3/2}(8n^2/\delta).
  \end{align*}
  In addition, any existing rule may generate at most $(m-1) (q(\sigma,n)+2) + q(\epsilon,n) +1$ mistakes (since it would be deleted after that). Note that $R = O(m\log(1/\delta))$, and $q(\epsilon,n) = O(\epsilon n + \log(n/\delta) + \sqrt{n \log(n/\delta)})$.
    The total upper bound is thus 
    
      $O\Big(\epsilon n + \sqrt{m n}\log^2(n/\delta)
      + m\log(1/\delta) (\epsilon n + m (\sigma n + \log(n/\delta) + \sqrt{n\log(n/\delta)}\Big).$
    
    Dividing by $n$ and reorganizing, we get the error rate in the statement of the lemma.    
\end{proof}

\bibliographystyle{abbrvnat}
\bibliography{sanjoy}

\begin{thebibliography}{16}
\providecommand{\natexlab}[1]{#1}
\providecommand{\url}[1]{\texttt{#1}}
\expandafter\ifx\csname urlstyle\endcsname\relax
  \providecommand{\doi}[1]{doi: #1}\else
  \providecommand{\doi}{doi: \begingroup \urlstyle{rm}\Url}\fi

\bibitem[Angluin and Kri{\c{k}}is(1994)]{AngluinKr94}
D.~Angluin and M.~Kri{\c{k}}is.
\newblock Learning with malicious membership queries and exceptions.
\newblock In \emph{Proceedings of the seventh annual conference on
  Computational learning theory}, pages 57--66. ACM, 1994.

\bibitem[Angluin et~al.(1997)Angluin, Kri{\c{k}}is, Sloan, and
  Tur{\'a}n]{AngluinKrSlTu97}
D.~Angluin, M.~Kri{\c{k}}is, R.~H. Sloan, and G.~Tur{\'a}n.
\newblock Malicious omissions and errors in answers to membership queries.
\newblock \emph{Machine Learning}, 28\penalty0 (2-3):\penalty0 211--255, 1997.

\bibitem[Branson et~al.(2010)Branson, Wah, Babenko, Schroff, Welinder, Perona,
  and Belongie]{BWBSWPB10}
S.~Branson, C.~Wah, B.~Babenko, F.~Schroff, P.~Welinder, P.~Perona, and
  S.~Belongie.
\newblock Visual recognition with humans in the loop.
\newblock In \emph{European Conference on Computer Vision}, 2010.

\bibitem[Croft and Das(1990)]{CD90}
W.~Croft and R.~Das.
\newblock Experiments with query acquisition and use in document retrieval
  systems.
\newblock In \emph{Proceedings of the 13th International Conference on Research
  and Development in Information Retrieval}, pages 349--368, 1990.

\bibitem[Dasgupta et~al.(2018)Dasgupta, Dey, Roberts, and Sabato]{DasguptaSa18}
S.~Dasgupta, A.~Dey, N.~Roberts, and S.~Sabato.
\newblock Learning from discriminative feature feedback.
\newblock In \emph{Advances in Neural Information Processing Systems}, pages
  3955--3963, 2018.

\bibitem[Druck et~al.(2008)Druck, Mann, and McCallum]{DMM08}
G.~Druck, G.~Mann, and A.~McCallum.
\newblock Learning from labeled features using generalized expectation
  criteria.
\newblock In \emph{Proceedings of ACM Special Interest Group on Information
  Retrieval}, 2008.

\bibitem[Kogan(1987)]{Kogan87}
A.~Y. Kogan.
\newblock Disjunctive normal forms of boolean functions with a small number of
  zeros.
\newblock \emph{USSR Computational Mathematics and Mathematical Physics},
  27\penalty0 (3):\penalty0 185--190, 1987.

\bibitem[{Mac Aodha} et~al.(2018){Mac Aodha}, Su, Chen, Perona, and
  Yue]{MSCPY18}
O.~{Mac Aodha}, S.~Su, Y.~Chen, P.~Perona, and Y.~Yue.
\newblock Teaching categories to human learners with visual explanations.
\newblock In \emph{IEEE Conference on Computer Vision and Pattern Recognition},
  2018.

\bibitem[Maximov(2013)]{Maximov13}
Y.~V. Maximov.
\newblock Implementation of boolean functions with a bounded number of zeros by
  disjunctive normal forms.
\newblock \emph{Computational Mathematics and Mathematical Physics},
  53\penalty0 (9):\penalty0 1391--1409, 2013.

\bibitem[Mubayi et~al.(2006)Mubayi, Tur{\'a}n, and Zhao]{MubayiTuZh06}
D.~Mubayi, G.~Tur{\'a}n, and Y.~Zhao.
\newblock The dnf exception problem.
\newblock \emph{Theoretical computer science}, 352\penalty0 (1-3):\penalty0
  85--96, 2006.

\bibitem[Poulis and Dasgupta(2017)]{PD17}
S.~Poulis and S.~Dasgupta.
\newblock Learning with feature feedback.
\newblock In \emph{Twentieth International Conference on Artificial
  Intelligence and Statistics}, 2017.

\bibitem[Raghavan et~al.(2005)Raghavan, Madani, and Jones]{RMJ05}
H.~Raghavan, O.~Madani, and R.~Jones.
\newblock Interactive feature selection.
\newblock In \emph{Proceedings of the 19th International Joint Conference on
  Artificial Intelligence}, pages 841--846, 2005.

\bibitem[Settles(2011)]{S11a}
B.~Settles.
\newblock Closing the loop: fast, interactive semi-supervised annotation with
  queries on features and instances.
\newblock In \emph{Empirical Methods in Natural Language Processing}, 2011.

\bibitem[Visotsky et~al.(2019)Visotsky, Atzmon, and Chechik]{VisotskyAtCh19}
R.~Visotsky, Y.~Atzmon, and G.~Chechik.
\newblock Learning with per-sample side information.
\newblock In \emph{AGI}, 2019.

\bibitem[Zhuravlev(1985)]{Zhuravlev85}
Y.~I. Zhuravlev.
\newblock Realization of boolean functions with a small number of zeros by
  disjunctive normal forms and related problems.
\newblock \emph{Soviet Mathematics-Doklady}, 32\penalty0 (3):\penalty0
  771--775, 1985.

\bibitem[Zou et~al.(2015)Zou, Chaudhuri, and Kalai]{ZCK15}
J.~Zou, K.~Chaudhuri, and A.~T. Kalai.
\newblock Crowdsourcing feature discovery via adaptively chosen comparisons.
\newblock In \emph{Conference on Human Computation and Crowdsourcing (HCOMP)},
  2015.

\end{thebibliography}
\end{document}